\documentclass[letterpaper]{article} 
\usepackage{aaai2027}  
\usepackage[hyphens]{url}  
\usepackage{graphicx} 
\usepackage{natbib}  
\usepackage{caption} 
\usepackage{algorithm}
\usepackage{algorithmic}
\usepackage{amsmath}
\usepackage{amssymb}
\usepackage{tabularx}
\usepackage{times}
\usepackage{mathtools}
\usepackage{latexsym}
\usepackage{booktabs}
\usepackage{pifont}
\usepackage{graphicx}
\usepackage{xcolor}
\newcommand{\cmark}{\ding{51}}%
\newcommand{\xmark}{\ding{55}}%

\usepackage{amsmath,amssymb,amsfonts,amsthm}

\usepackage{newfloat}
\usepackage{listings}
\DeclareCaptionStyle{ruled}{labelfont=normalfont,labelsep=colon,strut=off} 
\floatstyle{ruled}
\newfloat{listing}{tb}{lst}{}
\floatname{listing}{Listing}
\newtheorem{assumption}{Assumption}
\newtheorem{theorem}{Theorem}
\newtheorem{remark}{Remark}
\newtheorem{corollary}{Corollary}
\newtheorem{lemma}{Lemma}
\newtheorem{proposition}{Proposition}

\usepackage{booktabs}

\title{Personalized Federated Direct Preference Optimization via LoRA Adapters}
\author{
    Zhanhong Jiang\textsuperscript{\rm 1}\corresponding, Soumik Sarkar\textsuperscript{\rm 1, \rm 2}\\
}
\affiliations{
    \textsuperscript{\rm 1}Translational AI Center, Iowa State University, Ames, IA 50011 USA \\
    \textsuperscript{\rm 2}Department of Mechanical Engineering, Iowa State University, Ames, IA 50011 USA \\

    \{zhjiang, soumiks\}@iastate.edu
}

\begin{document}

\maketitle

\begin{abstract}
Federated Direct Preference Optimization (FedDPO) enables privacy-preserving alignment of language models across distributed agents without centralizing sensitive preference data. However, existing FedDPO analysis assumes all clients converge to a single shared policy, which is fundamentally at odds with the heterogeneous preference distributions present in practice. We propose \textbf{Personalized FedDPO (pFedDPO)}, in which each client maintains a small private LoRA adapter atop a shared global base model. In pFedDPO, the base model is synchronized via FedAvg each round, while LoRA adapters remain local and capture client-specific preference patterns. We prove that this decomposition replaces the gradient diversity $\kappa^2$ in the FedDPO bound with a residual heterogeneity $\kappa^2_{\mathrm{eff}}$, which is no larger than $\kappa^2$ when adapters absorb client-specific directions, at the cost of an additive adapter-bias term that vanishes as adapters converge. Experiments on preference datasets show that pFedDPO consistently attains a lower gradient norm than vanilla FedDPO while remaining competitive in training loss and per-client reward accuracy.
\end{abstract}


\section{Introduction}

Aligning language models (LMs)~\cite{wang2023self} with human preferences is increasingly performed via Direct Preference Optimization (DPO)~\cite{rafailov2023direct}, which avoids explicit reward modeling by optimizing a classification-style loss over (chosen/rejected) response pairs. As AI systems proliferate to distributed and privacy-sensitive environments, e.g., healthcare devices~\cite{nguyen2022federated}, personal assistants~\cite{sa2025cefeel}, institutional silos~\cite{huang2023promoting}, preference data is naturally fragmented across users who cannot share their data due to regulatory constraints (HIPAA~\cite{moore2019review}, GDPR~\cite{li2019impact}). This motivates distributed DPO, and in particular the federated and decentralized variants recently analyzed theoretically in~\cite{jiang2026distributed}. A critical assumption for FedDPO in~\cite{jiang2026distributed} is that all $N$ clients are driven toward a common global policy. This is appropriate when preference distributions are homogeneous, but unrealistic in practice: a medical professional in an ICU and a casual user evaluating restaurant suggestions have genuinely different preferences that should not be averaged away. The standard FedDPO convergence bound contains an irreducible floor proportional to gradient diversity ($\kappa^2$), which captures exactly this inter-client heterogeneity and cannot be reduced by communicating more or training longer, for a fixed step size.

Parameter-efficient fine-tuning, and LoRA~\cite{hu2022lora} in particular, offers a natural solution: decompose each client’s model into a shared global base and a small private adapter. The base model learns the shared structure of preferences across all agents and is aggregated via FedAvg~\cite{li2019convergence}, while the adapter captures client-specific deviations and is never communicated. This decomposition is the key contribution of this work, differing from the method in~\cite{zhu2026fedpdpo}.

We make the following contributions:
We propose pFedDPO, a personalized federated DPO algorithm in which each client maintains a LoRA adapter atop a shared base model updated via FedAvg. We prove a convergence bound for pFedDPO in which $\kappa^2$ is replaced by $\kappa^2_{\mathrm{eff}}$ at the cost of an explicit adapter-bias term $C\varepsilon_\phi$, so the bound is tighter than vanilla FedDPO whenever the heterogeneity reduction outweighs this bias (Theorem~\ref{theorem_1}). On two preference datasets, pFedDPO consistently attains a lower gradient norm than vanilla FedDPO with competitive loss and per-client accuracy.


\section{Background and Problem Setup}
\noindent\textbf{Federated DPO.}
We follow the federated setting of~\cite{jiang2026distributed}. $N$ clients index $i\in[N]=\{1,2,...,N\}$ each hold a disjoint dataset $\mathcal{D}_i$ of preference pairs $(\tau^+,\tau^-)$. The per-client DPO loss is
\begin{equation}
\mathcal{L}^i_\theta=\mathbb{E}_{(\tau^+,\tau^-)\sim\mathcal{D}_i}[l_\theta(\tau^+,\tau^-)]
\end{equation}
where $l_\theta (\tau^+,\tau^-) = -\text{log}\sigma(\beta[\Delta_\theta(\tau^+,\tau^-)-\Delta_\text{ref}(\tau^+,\tau^-)])$, 
$\beta>0$ is the inverse temperature parameter, $\Delta_\theta(\tau^+,\tau^-):=\text{log}\pi_\theta(\tau^+)-\text{log}\pi_\theta(\tau^-)$, and $\Delta_\text{ref}(\tau^+,\tau^-):=\text{log}\pi_\text{ref}(\tau^+)-\text{log}\pi_\text{ref}(\tau^-)$, $\pi_\text{ref}$ is a reference policy.
The global objective is $\mathcal{L}_\theta=(1/N)\sum_i\mathcal{L}_\theta^i$. FedDPO runs $E\geq 1$ local gradient steps per client, aggregates via FedAvg, and repeats for $R$ rounds. The convergence bound (Theorem 5.1 of~\cite{jiang2026distributed}) contains an irreducible term proportional to $\kappa^2$, the gradient diversity: $\kappa^2:=(1/N)\sum_i\mathbb{E}[\|\nabla_\theta \mathcal{L}_\theta-\nabla_\theta\mathcal{L}^i_\theta\|^2]$, where $\|\cdot\|$ indicates the Euclidean norm. This term reflects how much each client's local gradient deviates from the global gradient, and is driven entirely by \textit{preference heterogeneity} across clients. It cannot be eliminated by algorithmic choices within vanilla FedDPO.

\noindent\textbf{LoRA.}
Low-Rank Adaptation (LoRA)~\cite{hu2022lora} reparameterizes a weight matrix $W\in\mathbb{R}^{m\times n}$ as $W+BA$, with $B\in\mathbb{R}^{m\times r}$, $A\in\mathbb{R}^{r\times n}$, $r\ll\min(m,n)$. Standard LoRA freezes $W$; in pFedDPO, $W$ belongs to the shared base $\theta$ and is also trained. For a $768\times768$ matrix and $r=8$, the adapter has 12,288 parameters versus 589,824. Adapters add no communication since they are never transmitted.
\section{Personalized FedDPO (pFedDPO)}
\textbf{Algorithm.}
Each client $i$ maintains a model $\pi_{\theta,\phi_i}$, where $\theta\in\mathbb R^d$ is the shared global base and $\phi_i=(A_i,B_i)$ is a private LoRA adapter. The effective policy is $\pi_\theta$ with its attention weights augmented by $B_iA_i$. We write $\mathcal{L}^i_{\theta,\phi_i}$ for client $i$'s loss (1) under $\pi_{\theta,\phi_i}$. The global base $\theta$ is synchronized via FedAvg every round, and adapters $\phi_i$ are never communicated. Thus, in each round $r$:
\begin{itemize}
    \item Server broadcasts current global base $\theta^r$ to selected clients $\mathcal{S}_r\subset [N]$, $|\mathcal{S}_r|=S$.
    \item Each client $i\in\mathcal{S}_r$ loads $\theta^r$, restores its private $\phi_i$ from local storage, and runs $E$ local DPO gradient steps updating both $\theta$ and $\phi_i$ jointly.
    \item Client uploads only the base model update $\Delta\theta_i=\theta_i^{r,E}-\theta^r$ (adapter $\phi_i$ remains private).
    \item Server aggregates $\theta^{r+1}=(1/|\mathcal{S}_r|)\sum_{i\in\mathcal{S}_r}(\theta^r+\Delta\theta_i)$.
\end{itemize}
The client's effective policy at inference is $\pi_{\theta^r,\phi_i}$, \textit{personalizing} to $\mathcal{D}_i$ while leveraging the shared structure learned across all agents. Each selected client uploads exactly one vector $\Delta\theta_i\in \mathbb{R}^{d}$ per round, the same dimension as vanilla FedDPO. The per-round communication volume is therefore $S\times d$ parameters in both cases. pFedDPO's advantage is not reduced communication volume but improved convergence quality: $\Delta\theta_i$ reflects gradient work on a base model partially decoupled from client-specific preference directions by $\phi_i$, yielding smaller effective heterogeneity $\kappa^2_\text{eff}$ (defined below) in the convergence bound. Please see Appendix for the motivation why LoRA adapters are kept private. Figure~\ref{fig:pfeddpo_diagram} shows the specific detail of how pFedDPO works schematically.
\begin{figure*}
    \centering
    \includegraphics[width=0.8\linewidth]{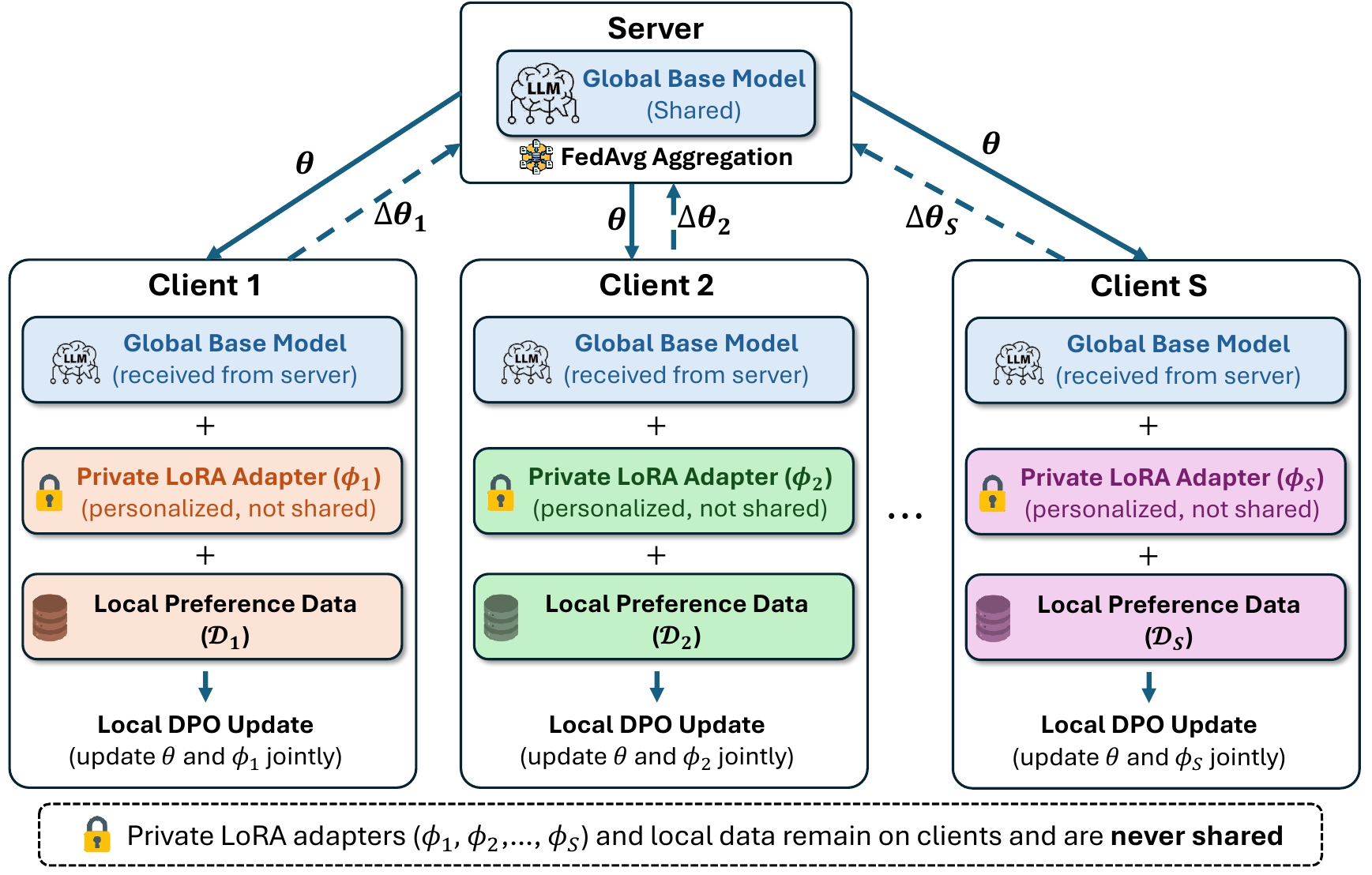}
    \caption{Illustration of the pFedDPO framework. The server maintains a shared global base model $\theta$. Each client $i$ receives $\theta$ from the server, combines it with its private LoRA adapter $\phi_i$ (never shared), and performs local DPO updates on its private dataset $\mathcal{D}_i$. Only the base model updates $\Delta\theta_i$ are uploaded to the server for FedAvg aggregation, while adapters and data remain local. This design is intended to reduce effective gradient diversity ($\kappa^2_\text{eff}\leq \kappa^2$ under the conditions in the Appendix) and enables personalization without additional communication overhead.}
    \label{fig:pfeddpo_diagram}
\end{figure*}

\noindent\textbf{Convergence Analysis.}
We analyze pFedDPO under the same assumptions as~\cite{jiang2026distributed} (Assumptions 3.4-4.1). The key insight is that \textit{the LoRA adapter absorbs a portion of the inter-client gradient divergence, reducing the effective heterogeneity seen by the base model update}.
Let $F_i(\theta):=\min_\phi\mathcal{L}^i_{\theta,\phi}$ and $F:=\frac1N\sum_iF_i$. By the envelope theorem, $\nabla F_i(\theta)=\nabla_\theta\mathcal{L}^i_{\theta,\phi_i^*}$ with $\phi_i^*\in\arg\min_\phi\mathcal{L}^i_{\theta,\phi}$, i.e., the base gradient after client $i$'s adapter has absorbed what it can express. We define the residual heterogeneity $\kappa^2_{\mathrm{eff}}:=\frac1N\sum_i\mathbb{E}[\|\nabla F(\theta)-\nabla F_i(\theta)\|^2]$. Since the adapter removes client-specific directions from the base gradient, $\kappa^2_{\mathrm{eff}}\le\kappa^2$ to first order under a shared-subspace condition (Appendix); we also track it empirically.
To characterize the main result, we impose an assumption in the following.
\begin{assumption}[Adapter Approximate Stationarity]\label{assumption_1}
Along the local iterates, $\mathbb{E}[\|\nabla_\phi\mathcal{L}^i_{\theta,\phi_i}\|^2]\le\varepsilon_\phi$; moreover, $\mathcal{L}^i_{\theta,\cdot}$ is $\mu_\phi$-PL in $\phi$ and $\nabla_\theta\mathcal{L}^i_{\theta\phi}$ is $L_{\theta,\phi}$-Lipschitz in $\phi$.
\end{assumption}
Because adapters are trained only $E$ steps per round, $\nabla_\theta\mathcal{L}^i_{\theta,\phi_i}$ is a biased estimate of $\nabla F_i(\theta)$. Assumption~1 bounds this bias by $(L_{\theta\phi}/\mu_\phi)^2\varepsilon_\phi$, giving $C\varepsilon_\phi$ with $C=\mathcal{O}(L_{\theta\phi}^2/\mu_\phi^2)$.
\begin{theorem} (pFedDPO Convergence)\label{theorem_1} Under Assumptions 3.4-4.1 of \cite{jiang2026distributed} and with LoRA rank $r\geq 1$, let $\eta\leq \text{min}\{1/(8LE), S/(16LNE)\}$, where $L>0$ is smoothness constant of $F$, then after $R$ rounds of pFedDPO:
    \begin{equation}\label{eq_2}
    \begin{split}
        &\frac{1}{R}\sum_{r=0}^{R-1}\mathbb E[\|\nabla F(\theta^r)\|^2]\leq \frac{2(F(\theta^0)-F^*)}{\eta ER} + \frac{8L\eta\zeta^2_g}{S}\\&+16L^2\eta^2E\kappa^2_\text{eff} + \frac{16L^2\eta^2E^2\zeta^2_g}{S}+C\varepsilon_\phi,
    \end{split}
    \end{equation}
where $\zeta^2_g$ is the variance of gradients, $F^*:=\text{inf}_\theta F(\theta)$.
\end{theorem}
\noindent\textbf{\textit{Proof sketch}.} We apply Thm.~5.1 of (Jiang 2026) to $F$. Each client's base gradient decomposes as $\nabla_\theta\mathcal{L}^i_{\theta,\phi_i}=\nabla F_i(\theta)+b_i$ with $\|b_i\|^2\le(L_{\theta\phi}/\mu_\phi)^2\varepsilon_\phi$ (Assumption~1). The heterogeneity of $\{\nabla F_i\}$ is $\kappa^2_{\mathrm{eff}}$, so Lemmas C.1--C.5 of (Jiang 2026) apply with $\kappa^2\to\kappa^2_{\mathrm{eff}}$, and the bias adds $C\varepsilon_\phi$ after telescoping. Since adapters are never aggregated, $\nabla_\phi\mathcal{L}^i$ does not enter the drift. Full proof in Appendix.


Compared to vanilla FedDPO (Theorem 5.1 of \cite{jiang2026distributed}), $\kappa^2$ is replaced by $\kappa^2_\text{eff}$ (improvement), and $C\varepsilon_\phi\geq 0$ is added (cost). The net bound is tighter than vanilla FedDPO whenever the reduction $16L^2\eta^2E(\kappa^2-\kappa^2_\text{eff})$ exceeds $C\varepsilon_\phi$. We expect both the improvement and the cost to grow with $r$: higher $r$ can reduce $\kappa^2_\text{eff}$ further, but larger adapters may be harder to converge in $E$ steps, increasing $\varepsilon_\phi$. 
This suggests a non-monotone relationship between $r$ and performance, with an optimal rank balancing the two effects.
\section{Experiments}
We evaluate on Stanford SHP~\cite{esfandiari2021cross} (main result) and Anthropic hh-rlhf~\cite{bai2022training} (in Appendix), each with $N = 5$ agents holding 120 disjoint non-IID preference pairs. The backbone is DistilGPT-2 ($\approx$82M parameters). LoRA targets the attention projection (``\texttt{c\_attn}'') with default rank $r = 8$ and $\alpha = 16$ (the scaling parameter). We compare three methods: Centralized DPO (pooled data), Vanilla FedDPO (full model aggregation, $S=5$), and pFedDPO (base + private LoRA, $S=5$). All methods use AdamW ($\eta=10^{-4}$), $\beta=0.2$, gradient clipping at 1.0, batch size 4, $E = 3$ local steps, and $R = 80$ communication rounds. A full list is in Appendix.

\noindent\textbf{Comparative study.} As Figure~\ref{fig:comparative_shp} shows, centralized DPO (dashed grey) shows the highest gradient norm, highest and most unstable training loss (oscillating between 0.0 and 1.4), and the lowest preference accuracy ($\sim$0.75 at round 80, accuracy on each agent's training pairs). Both vanilla FedDPO and pFedDPO substantially outperform it. This is not a failure of centralized training in principle, but a consequence of the experimental regime: centralized DPO processes batch size 4 per gradient step, while each FedDPO round aggregates across $S=5$ agents (effective batch 20). This places the experiment in the computation-dominated regime of Theorem 5.1 in~\cite{jiang2026distributed}, where larger effective local computation per round accelerates convergence. Centralized DPO is therefore not an oracle in this setup. 
It thus serves as a single-learner reference, consistent with the $\zeta^2_g/S$ term in Theorem~\ref{theorem_1}.
pFedDPO attains a lower gradient norm than vanilla FedDPO with comparable loss and accuracy.
\begin{figure}
    \centering
    \includegraphics[width=1\linewidth]{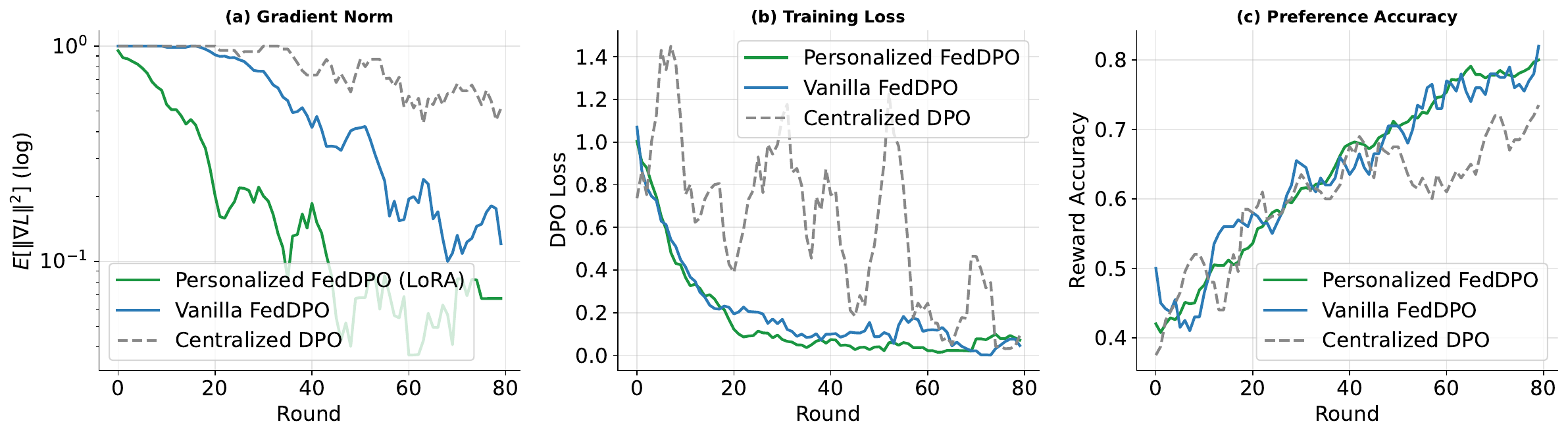}
    \caption{Comparison of baselines with SHP data.}
    \label{fig:comparative_shp}
\end{figure}
\begin{figure}
    \centering
    \includegraphics[width=1\linewidth]{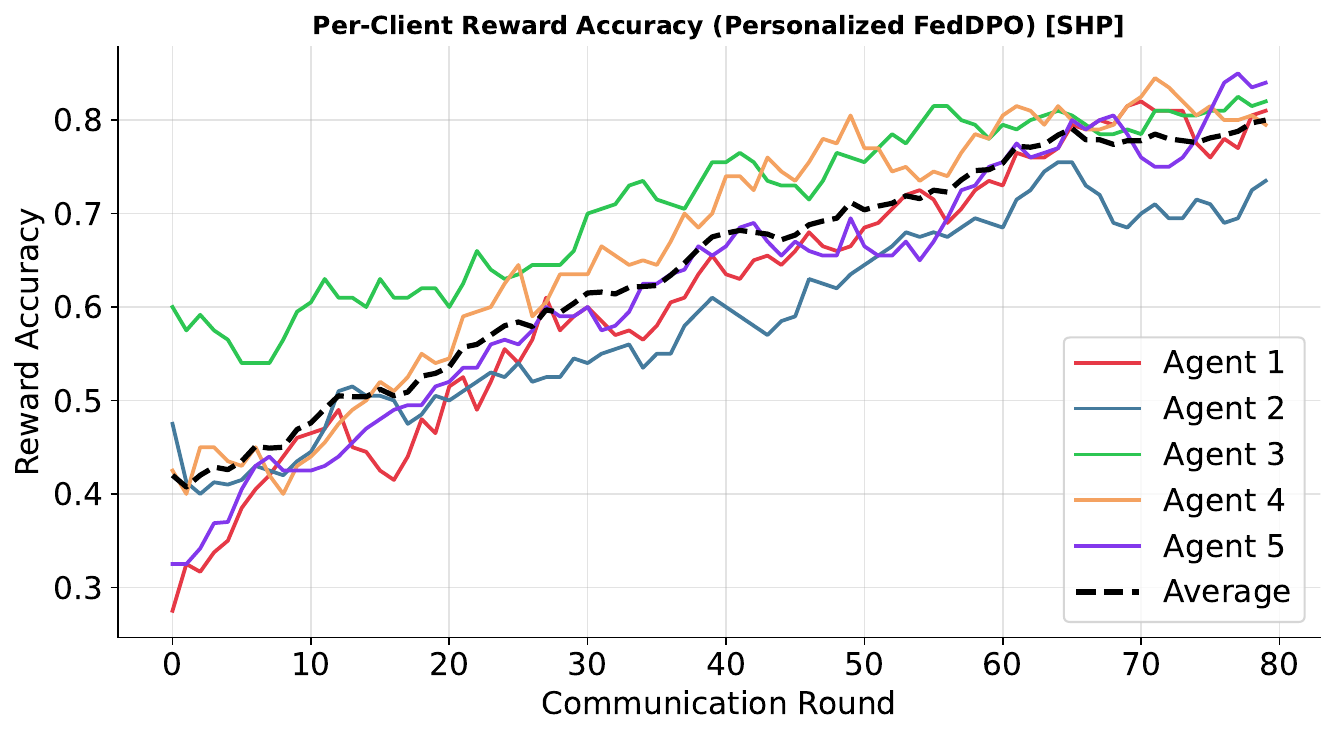}
    \caption{Per-client preference accuracy with SHP data.}
    \label{fig:per_client_shp}
\end{figure}
Additionally, Figure~\ref{fig:per_client_shp} depicts that all five agents improve from $\sim$0.30–0.60 at round 0 to $\sim$0.75–0.85 by round 80 under pFedDPO. The spread at round 80 is roughly 10 percentage points (Agent 5 highest at $\sim$0.85, Agent 2 lowest at $\sim$0.75), consistent with heterogeneous preferences across agents rather than a single dominant agent driving the average. The mean accuracy (dashed) reaches approximately 0.80.

\begin{figure}
    \centering
    \includegraphics[width=1\linewidth]{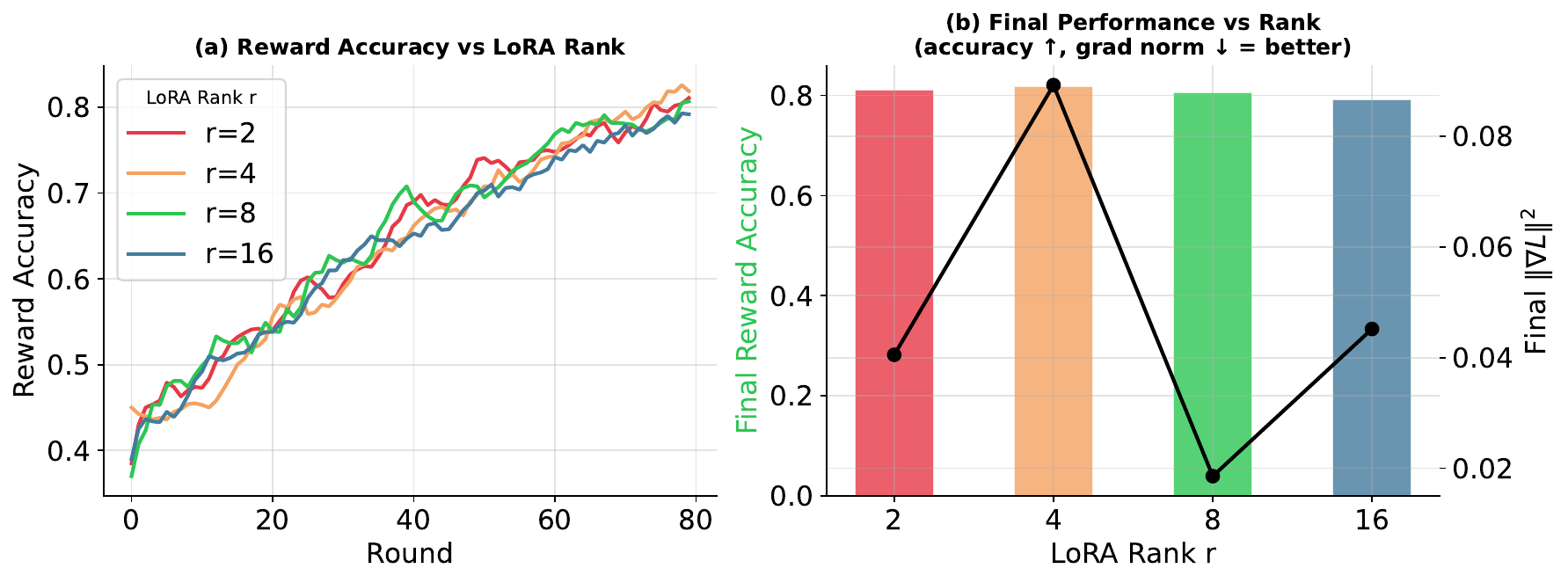}
    \caption{Low rank $r$ ablation with SHP data.}
    \label{fig:rank_shp}
\end{figure}
\noindent\textbf{LoRA Rank Ablation.} Figure~\ref{fig:rank_shp} ablates LoRA rank $r\in\{2, 4, 8, 16\}$. In panel (a), reward accuracy curves for all four ranks are broadly similar ($\sim$0.78–0.82 at round 80), with modest differences across ranks. Panel (b) reveals a more nuanced picture: $r = 4$ achieves the highest final reward accuracy ($\sim$0.82) but with the highest gradient norm, while $r = 8$ achieves the lowest final gradient norm ($\sim$0.02). $r = 16$ shows a higher gradient norm than $r = 8$, and $r = 16$ has the lowest bar in final accuracy. This non-monotone relationship between rank and accuracy is consistent with the trade-off suggested by Theorem~\ref{theorem_1}: higher rank reduces $\kappa^2_\text{eff}$ (improving convergence) but increases the adapter-bias term $C\varepsilon_\phi$ (harder to converge a larger adapter in $E = 3$ local steps). The optimal rank balances these two competing effects, and the data suggests $r = 4$ or $r = 8$ is near-optimal for this model and data regime. We use $r = 8$ as the default since it achieves the best gradient norm, which is the primary metric in our theoretical analysis. 

\begin{figure}
    \centering
    \includegraphics[width=1\linewidth]{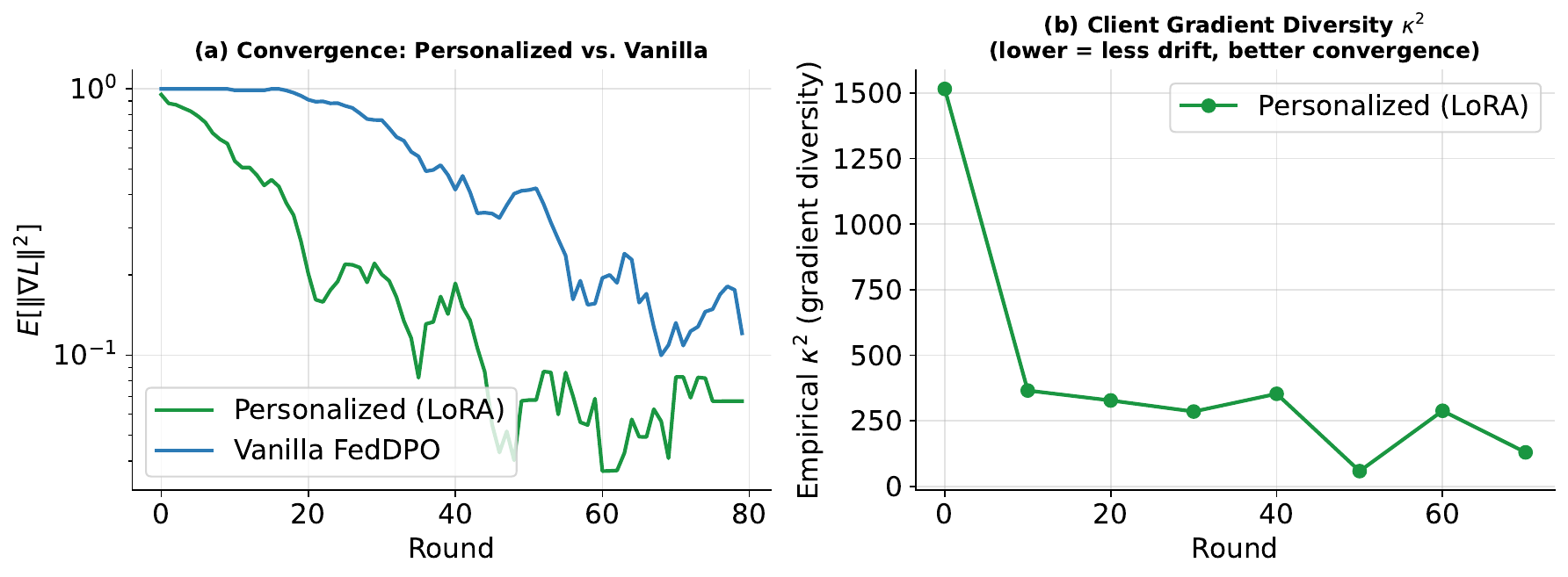}
    \caption{Gradient diversity $\kappa^2$ ablation with SHP data.}
    \label{fig:diversity_shp}
\end{figure}

\begin{figure}
    \centering
    \includegraphics[width=1\linewidth]{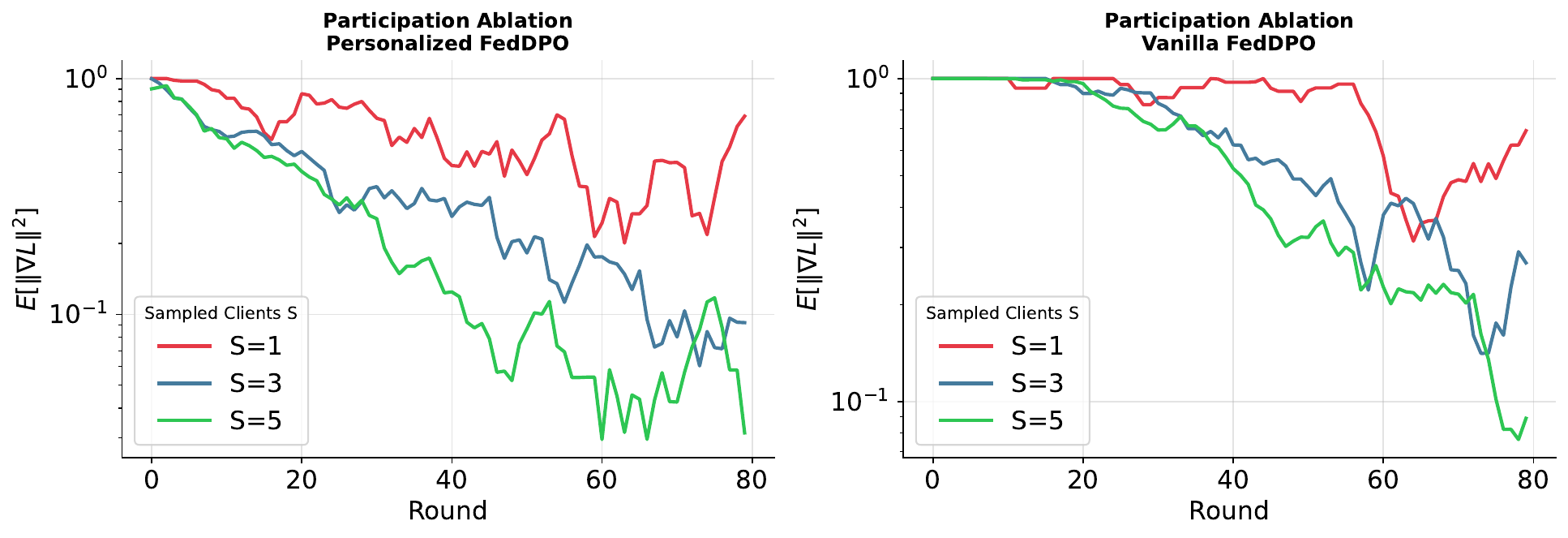}
    \caption{Participation $S$ ablation with SHP data.}
    \label{fig:participation_shp}
\end{figure}

\noindent\textbf{Gradient Diversity and Participation Ablation.} As shown in Figure~\ref{fig:diversity_shp}(a), pFedDPO consistently achieves a lower gradient norm than vanilla FedDPO across all 80 rounds, with the gap largest in the later rounds as adapters accumulate client-specific knowledge. Figure~\ref{fig:diversity_shp}(b) shows the empirical $\kappa^2$ measured every 10 rounds on pFedDPO's base model gradients. Starting from $\sim$1500 at round 0, $\kappa^2$ drops sharply to $\sim$350 by round 10 and continues to decrease gradually to approximately 150–200 by round 70. 
The decrease is not monotone (minor fluctuations around rounds 40--50) and partly reflects overall training progress. At initialization the adapters are zero ($B_i=0$), so the starting value coincides with vanilla FedDPO's $\kappa^2$. The $\sim$85--90\% reduction is consistent with the $\kappa^2\to\kappa^2_{\mathrm{eff}}$ mechanism of Theorem~\ref{theorem_1}.
Figure~\ref{fig:participation_shp} shows the participation ablation $S \in\{1, 3, 5\}$ for both pFedDPO (left) and vanilla FedDPO (right). In both cases the ordering $S = 5 > S = 3 > S = 1$ holds, consistent with the $\zeta_g^2/S$ variance floor in Theorem 5.1 of \cite{jiang2026distributed}. At $S = 1$, both methods struggle: the gradient norm remains near 1.0 throughout with high variance for pFedDPO, and similarly for vanilla FedDPO. The advantage of personalization is clearest at $S = 5$ and $S = 3$, where pFedDPO's lower gradient norm floor is visible. At $S = 1$ the $\zeta^2_g/S$ variance term dominates and the personalization benefit from $\kappa^2_\text{eff}$ is masked.

\section{Related Work}
DPO~\cite{rafailov2023direct} and its extensions~\cite{amini2024direct,park2024disentangling,xiao2024cal} focus on the centralized, single-model setting. FedDPO and DecDPO (decentralized DPO) were proposed and analyzed in~\cite{jiang2026distributed}, establishing the first convergence guarantees for preference optimization in distributed settings, where prior work focuses primarily on learning~\cite{saadati2024dimat,saadati2025foundation,esfandiari2021cross} and more recently agentic AI systems~\cite{rillo2026adko}. Our work extends this by addressing the personalization gap they identify as a limitation. Existing federated DPO approaches~\cite{spadea2025federated,zhu2026fedpdpo} do not address this theoretical gap: they either treat all clients symmetrically without personalization~\cite{spadea2025federated} or introduce personalization empirically without bounding how it improves convergence~\cite{zhu2026fedpdpo}.
Personalized federated learning with LoRA has been studied in~\cite{yi2023pfedlora,qi2024fdlora,song2026pfl,gu2026d,lee2026fedp2eft}, but not in the context of preference optimization, where the trajectory-level DPO objective introduces additional variance and heterogeneity that require explicit analysis. We refer interested readers to Appendix for more related work and include a methodological comparison table.

\section{Conclusion}
We propose pFedDPO, a personalized federated DPO framework where each client trains a private LoRA adapter on top of a shared global base model. The base model is synchronized across clients each round, while adapters remain local, preserving personalization without increasing communication compared with standard FedDPO. Our analysis shows this design better handles client heterogeneity while introducing only a small bias that diminishes over training. Empirically, pFedDPO consistently lowers the gradient norm relative to vanilla FedDPO while keeping preference accuracy competitive across all clients. 

\section*{Acknowledgments}
This work was supported by the COALESCE: COntext Aware LEarning for Sustainable CybEr-Agricultural Systems (NSF Frontier \#1954556).

\bibliography{aaai2027}

\clearpage

\section{Additional Related Work}\label{sec:additional_related_work}
FedDPO~\cite{jiang2026distributed} establishes the first convergence guarantees for distributed DPO under full and partial participation, staleness, and decentralized topologies, but targets a single shared policy and does not support personalization. Its bound contains an irreducible $\kappa^2$ floor that our work directly reduces.

Spadea and Seneviratne~\cite{spadea2025federated} compare Kahneman-Tversky Optimization (KTO) and DPO in a federated LLM fine-tuning setting using Alpaca-7B. They show that KTO outperforms DPO across MT-Bench-1, Vicuna, and AdvBench benchmarks, particularly under redistributed (non-paired) data where DPO cannot be applied. Their work is purely empirical with no convergence analysis, and does not address personalization: all clients train toward a single global model. The key contribution is demonstrating KTO's resilience to non-IID data due to its single-response feedback requirement, which sidesteps the paired-response constraint of DPO.

FedPDPO~\cite{zhu2026fedpdpo} uses a globally shared LoRA adapter with a personalized client-specific LLM head, a client-specific explicit reward head, and a bottleneck adapter to balance global and local features. Its theory is described as establishing a probabilistic foundation and soundness.
It transmits shared LoRA updates of dimension $\mathcal{O}(r(m+n))$ per client, achieving up to 4.80\% average accuracy improvement. However, its adapters are globally shared rather than private, and its theory does not provide a rate-explicit bound relating personalization to $\kappa^2$. pFedDPO makes the opposite trade-off: full $d$-dimensional base updates (same volume as vanilla FedDPO) in exchange for fully private adapters and an explicit convergence guarantee.

The key differentiator of pFedDPO is therefore with a rate-explicit convergence bound that makes the effect of personalization explicit. Unlike FedPDPO, pFedDPO never transmits adapter weights, keeping all personalization parameters on-device, at a higher upload volume than FedPDPO. Unlike Spadea and Seneviratne's work~\cite{spadea2025federated}, pFedDPO explicitly personalizes to each client’s preference distribution. Please see Table~\ref{tab:fed_comparison} for methodological comparison.

\section{Discussion on pFedDPO}\label{sec:diagram}
Figure~\ref{fig:pfeddpo_diagram} shows the specific detail of how pFedDPO works schematically. A closely related line of work~\cite{cho2024heterogeneous,saadati2025foundation,kuo2024federated,sun2024improving,yi2023pfedlora,bian2025lora,bian2026fedalt,xie2026fedlodrop} applies parameter-efficient federated fine-tuning by training LoRA adapters locally and periodically aggregating or averaging those adapters across clients. In contrast, \textit{pFedDPO intentionally keeps LoRA adapters private and only synchronizes the shared base model}. This design is motivated by the nature of preference alignment, where client heterogeneity reflects genuine differences in user values, style, or alignment objectives rather than noise to be averaged away.

When LoRA adapters are shared globally, they must simultaneously encode both:
\begin{itemize}
    \item task-level adaptation of the shared model, and
    \item client-specific preference information.
\end{itemize}
Under heterogeneous preference distributions, these objectives can conflict. Averaging adapters across clients may therefore dilute personalized signals, causing \textit{preference interference}: updates that improve alignment for one client can harm another. This is especially pronounced in DPO, where gradients are driven by pairwise preference comparisons and are highly sensitive to local preference distributions. pFedDPO avoids this by separating these roles:
\begin{itemize}
    \item the base model captures globally transferable alignment knowledge shared across clients;
    \item the private LoRA adapter captures client-specific residual preferences.
\end{itemize}
This decomposition makes the global model responsible for learning common alignment structure, while the adapters absorb personalization. As a result, the server never aggregates incompatible preference-specific parameters.

A second advantage is privacy. LoRA adapters trained on preference data can encode highly sensitive information about user behavior, style, or preference distributions. Uploading them, even though lightweight, may leak personalization signals. Keeping adapters local provides a stronger privacy boundary than methods that transmit adapters to the server. From an optimization perspective, private adapters also reduce effective client drift. Since some heterogeneity is absorbed locally, the shared model sees less conflicting gradient signal during aggregation, which leads to more stable federated optimization. This is reflected in our analysis through a smaller effective heterogeneity term relative to vanilla FedDPO~\cite{jiang2026distributed}. Finally, communication remains unchanged relative to standard FedDPO: only updates to the shared base model are transmitted. Compared with adapter-sharing methods, pFedDPO therefore offers a better tradeoff between:
\begin{itemize}
    \item global coordination (via shared base model),
    \item local personalization (via private adapters),
    \item privacy preservation, and
    \item robustness to preference heterogeneity.
\end{itemize}

\setlength{\tabcolsep}{3pt}
\begin{table*}[htbp]
\centering
\caption{Comparison of federated fine-tuning methods. \cmark=yes; \xmark=no; $\sim$=partial. Comm.Vol.=per-round communication volume per selected client. pFedDPO transmits $d$-dimensional base updates (same as FedDPO), not LoRA dimensions. Its distinguishing property is a rate-explicit bound capturing how private adapters reduce $\kappa^2$ to $\kappa^2_\text{eff}$.}
\footnotesize  
\begin{tabularx}{\textwidth}{XXXXXXX}
\toprule
Method & Base Model & Personalization & Conv. Theory & Comm. Vol. & Feedback & Aggregation \\
\midrule
\cite{jiang2026distributed} 
  & Full model 
  & \xmark\ (shared policy) 
  & \cmark\ rate bound ($\kappa^2$ floor) 
  & $d$ per client 
  & Paired 
  & FedAvg (full model) \\
\addlinespace
\cite{spadea2025federated} 
  & Full model (Alpaca-7B) 
  & \xmark\ (uniform) 
  & \xmark\ (empirical only) 
  & $d$ per client 
  & KTO: single; DPO: paired 
  & FedAvg (full model) \\
\addlinespace
\cite{zhu2026fedpdpo}
  & Shared LoRA + bottleneck adapter 
  & \cmark\ (private LLM head + reward head) 
  & \textasciitilde\ (probabilistic soundness only) 
  & $\mathcal{O}((m+n)r)$ per client (LoRA dims) 
  & Paired + explicit reward 
  & FedAvg on shared modules; heads private \\
\addlinespace
\textbf{pFedDPO (ours)} 
  & Base + private LoRA per client 
  & \cmark\ (private $\phi_i$, never shared) 
  & \cmark\ rate bound ($\kappa^2_\text{eff}$+$C\varepsilon_\phi$) 
  & $d$ per client (same as FedDPO) 
  & Paired 
  & FedAvg on base only; adapter stays local \\
\bottomrule
\end{tabularx}
\label{tab:fed_comparison}
\end{table*}


\section{Hyperparameters}\label{sec:hyperparameter}
The specific hyperparameters used in the empirical evaluation are listed in Table~\ref{tab:hyperparameters}.
\begin{table}[h]
    \caption{Hyperparameter Setup.}
    \centering
    \begin{tabular}{c|c}
        \toprule
        Hyperparameter & Value \\
        \midrule
        Model & DistilGPT-2 ($\sim82$ M)\\
        \# Agents (N) & 5 \\
        Training samples per agent  & 120 \\
        Inverse temperature $\beta$ & 0.2\\
        Optimizer & AdamW\\
        Learning rate $\eta$ & $1e^{-4}$\\
        Gradient clipping norm & 1.0\\
        Communication rounds & 80\\
        Local steps ($E$) & 3\\
        Max sequence length & 64 tokens\\
        Default participation $S$ & 5 (full)\\
        LoRA rank $r$ & 8\\
        LoRA $\alpha$ & 16\\
        LoRA dropout & 0.05\\
        LoRA target modules & \texttt{c\_attn}\\
        Comm.Volume (pFedDPO) & $S\times d=5\times d$\\
        \bottomrule
    \end{tabular}
    \label{tab:hyperparameters}
\end{table}
\section{Proof of Theorem~\ref{theorem_1}}\label{sec:proof_theorem_1}
The proof follows the structure of Theorem 5.1 in~\cite{jiang2026distributed}. We highlight only the modification due to LoRA personalization.
\begin{proof}
    Let $b_i^{r,e}:=\nabla_\theta\mathcal{L}^i_{\theta_i^{r,e},\phi_i^{r,e}}-\nabla F_i(\theta_i^{r,e})$.
(i) Bias: $\mu_\phi$-PL implies the error bound $\mathrm{dist}(\phi,\arg\min)\le\|\nabla_\phi\mathcal{L}^i\|/\mu_\phi$; with cross-Lipschitzness, $\mathbb{E}\|b_i^{r,e}\|^2\le(L_{\theta\phi}/\mu_\phi)^2\varepsilon_\phi$.
(ii) Descent: by $L$-smoothness of $F$ and $\langle a,a+b\rangle\ge\frac12\|a\|^2-\frac12\|b\|^2$, the bias adds $\frac{\eta E}{2}\bar b^2$ per round; dividing by $\eta ER/2$ gives $c_0(L_{\theta\phi}/\mu_\phi)^2\varepsilon_\phi=:C\varepsilon_\phi$.
(iii) Drift: Lemmas C.1--C.5 of (Jiang 2026) with $\nabla\mathcal{L}^i\to\nabla F_i$ give heterogeneity $\kappa^2_{\mathrm{eff}}$; the bias' drift contribution is $\mathcal{O}(L^2\eta^2E^2)\varepsilon_\phi$, absorbed into $C\varepsilon_\phi$ for $\eta\le1/(8LE)$.

\end{proof}
To facilitate the understanding of the proof in the above, we include another detailed analysis on why $\kappa^2_\text{eff}\leq \kappa^2$ and present the intuition behind it.

\section{Why $\kappa^2_{\mathrm{eff}} \le \kappa^2$: Technical Analysis and Intuition}
\label{app:kappa_analysis}

\subsection{Setup}
 
Partition the base parameters as $\theta=(\theta_{\mathcal{A}},\theta_{\mathcal{N}})$, where $\theta_{\mathcal{A}}=(W_1,\dots,W_M)$ collects the LoRA-adapted matrices $W_\ell\in\mathbb{R}^{m\times n}$ and $\theta_{\mathcal{N}}$ all remaining parameters. In our experiments, $\theta_{\mathcal{A}}$ consists of the $M=6$ \texttt{c\_attn} matrices of DistilGPT-2 ($m=768$, $n=2304$), and $\theta_{\mathcal{N}}$ contains the MLP, embedding, and LayerNorm parameters. Client $i$'s adapter is $\phi_i=\{(B_{i,\ell},A_{i,\ell})\}_{\ell=1}^M$ with $B_{i,\ell}\in\mathbb{R}^{m\times r}$, $A_{i,\ell}\in\mathbb{R}^{r\times n}$, and effective weights $W_\ell+sB_{i,\ell}A_{i,\ell}$, where $s=\alpha/r$.
 
\paragraph{Adapters as private low-rank shifts.}
Let $\Gamma_i\in\mathbb{R}^d$ place $sB_{i,\ell}A_{i,\ell}$ on block $\ell$ of $\theta_{\mathcal{A}}$ and zeros on $\theta_{\mathcal{N}}$. Then the adapted model is simply the base model evaluated at shifted weights:
\begin{equation}\label{app:eq:shift}
\mathcal{L}^i_{\theta,\phi_i}=\mathcal{L}^i_{\theta+\Gamma_i},\qquad
\nabla_\theta\mathcal{L}^i_{\theta,\phi_i}=\nabla\mathcal{L}^i_{\theta+\Gamma_i}.
\end{equation}
Vanilla FedDPO corresponds to $\Gamma_i=0$, which is also the standard LoRA initialization ($B_{i,\ell}=0$).
 
\paragraph{The two heterogeneity constants.}
Write $g_i:=\nabla\mathcal{L}^i_\theta$ for client $i$'s unadapted gradient and $\bar g:=\nabla\mathcal{L}_\theta=\frac1N\sum_ig_i$. For a family $\{v_i\}_{i=1}^N$ of random vectors, define the norm
\[
\big\|\{v_i\}\big\|_N:=\Big(\tfrac1N\textstyle\sum_{i=1}^N\mathbb{E}\|v_i\|^2\Big)^{1/2},
\]
which satisfies the triangle inequality. The vanilla gradient diversity is
\begin{equation}\label{app:eq:kappa}
\kappa^2=\big\|\{\bar g-g_i\}\big\|_N^2=\frac1N\sum_{i=1}^N\mathbb{E}\big[\|\nabla\mathcal{L}_\theta-\nabla\mathcal{L}^i_\theta\|^2\big].
\end{equation}
For pFedDPO, let $F_i(\theta):=\min_\phi\mathcal{L}^i_{\theta,\phi}$ and $F:=\frac1N\sum_iF_i$, with minimizer $\phi_i^*$ and shift $\Gamma_i^*$. We assume the minimizing shift $\Gamma_i^*$ is unique, so that $F_i$ is differentiable and, by the envelope theorem, $\nabla F_i(\theta)=\nabla_\theta\mathcal{L}^i_{\theta,\phi_i^*}=\nabla\mathcal{L}^i_{\theta+\Gamma_i^*}$. The pair $(B,A)$ itself is never unique, since $(BQ,Q^{-1}A)$ gives the same shift for any invertible $Q$. Nothing below depends on that choice. The residual heterogeneity is
\begin{equation}\label{app:eq:kappaeff}
\kappa^2_{\mathrm{eff}}=\big\|\{\nabla F-\nabla F_i\}\big\|_N^2=\frac1N\sum_{i=1}^N\mathbb{E}\big[\|\nabla F(\theta)-\nabla F_i(\theta)\|^2\big].
\end{equation}
Finally, we split $\kappa^2=\kappa^2_{\mathcal{A}}+\kappa^2_{\mathcal{N}}$ into the heterogeneity carried by adapted blocks and by non-adapted parameters, respectively.
 
\subsection{An Exact Identity at Adapter Stationarity}
 
For block $\ell$ of client $i$, let $P_{i,\ell}$ be the orthogonal projector onto the column space of $B_{i,\ell}$ and $Q_{i,\ell}=A_{i,\ell}^\top(A_{i,\ell}A_{i,\ell}^\top)^{+}A_{i,\ell}$ the projector onto the row space of $A_{i,\ell}$. The adapter's tangent space is
\[
\mathcal{T}_{i,\ell}:=\{B_{i,\ell}X+YA_{i,\ell}:\ X\in\mathbb{R}^{r\times n},\ Y\in\mathbb{R}^{m\times r}\},
\]
of dimension $r(m+n)-r^2$ when $B_{i,\ell},A_{i,\ell}$ have full rank $r$. It depends only on the column and row spaces, and hence is invariant to the reparameterization above. In the Frobenius inner product, $Z\perp\mathcal{T}_{i,\ell}$ if and only if there exists a matrix $Z$ such that $B_{i,\ell}^\top Z=0$ and $ZA_{i,\ell}^\top=0$. The orthogonal projectors onto $\mathcal{T}_{i,\ell}^{\perp}$ and $\mathcal{T}_{i,\ell}$ are
\begin{equation}\label{app:eq:proj}
\begin{split}
&\Pi^\perp_{i,\ell}(Z)=(I-P_{i,\ell})\,Z\,(I-Q_{i,\ell}),\\&
\Pi_{i,\ell}(Z)=P_{i,\ell}Z+(I-P_{i,\ell})ZQ_{i,\ell}.
\end{split}
\end{equation}
Let $\Pi_i$ act as $\Pi_{i,\ell}$ on each block of $\theta_{\mathcal{A}}$ and as zero on $\theta_{\mathcal{N}}$, and let $\Pi_i^\perp:=I-\Pi_i$, which is the identity on $\theta_{\mathcal{N}}$. Since $B_{i,\ell}A_{i,\ell}\in\mathcal{T}_{i,\ell}$, we have $\Gamma_i\in\mathrm{range}(\Pi_i)$.
 
\begin{lemma}[Stationarity identity]\label{app:lem:stat}
Let $G_{i,\ell}$ denote block $\ell$ of $\nabla_\theta\mathcal{L}^i_{\theta,\phi_i}$. Then $\nabla_{B_{i,\ell}}\mathcal{L}^i_{\theta,\phi_i}=s\,G_{i,\ell}A_{i,\ell}^\top$ and $\nabla_{A_{i,\ell}}\mathcal{L}^i_{\theta,\phi_i}=s\,B_{i,\ell}^\top G_{i,\ell}$. Consequently, if $\nabla_\phi\mathcal{L}^i_{\theta,\phi_i}=0$, then
\begin{equation}\label{app:eq:stat}
\nabla_\theta\mathcal{L}^i_{\theta,\phi_i}=\Pi_i^\perp\,\nabla_\theta\mathcal{L}^i_{\theta,\phi_i}.
\end{equation}
In particular, $\nabla F_i(\theta)=\Pi^\perp_{i}\nabla F_i(\theta)$ with $\Pi_i$ built from $\phi_i^*$. That is, the base gradient after adaptation has no component in the adapter's tangent space.
\end{lemma}
 
\begin{proof}
The loss depends on $W_\ell$, $B_{i,\ell}$, $A_{i,\ell}$ only through the effective weight $W_\ell+sB_{i,\ell}A_{i,\ell}$, so $\nabla_{W_\ell}\mathcal{L}^i=G_{i,\ell}$ and the chain rule gives the two adapter gradients. If both vanish, then $G_{i,\ell}Q_{i,\ell}=(G_{i,\ell}A_{i,\ell}^\top)(A_{i,\ell}A_{i,\ell}^\top)^{+}A_{i,\ell}=0$ and $P_{i,\ell}G_{i,\ell}=B_{i,\ell}(B_{i,\ell}^\top B_{i,\ell})^{+}(B_{i,\ell}^\top G_{i,\ell})=0$. Hence $G_{i,\ell}=(I-P_{i,\ell})G_{i,\ell}(I-Q_{i,\ell})=\Pi^\perp_{i,\ell}G_{i,\ell}$. On $\theta_{\mathcal{N}}$, $\Pi_i^\perp$ is the identity.
\end{proof}
 
\begin{remark}[Approximate stationarity and initialization]\label{app:rem:approx}
The two terms of $\Pi_{i,\ell}(G)$ in \eqref{app:eq:proj} are orthogonal. Let $\varsigma_r(\cdot)$ denote the $r$-th largest singular value. If $B_{i,\ell}$ and $A_{i,\ell}$ have full rank $r$, then
\[
\|\Pi_{i,\ell}G_{i,\ell}\|_F^2\le\frac{\|\nabla_{A_{i,\ell}}\mathcal{L}^i\|_F^2}{s^2\varsigma_r(B_{i,\ell})^2}+\frac{\|\nabla_{B_{i,\ell}}\mathcal{L}^i\|_F^2}{s^2\varsigma_r(A_{i,\ell})^2}.
\]
Summing over $\ell$ and applying Assumption~1 gives $\mathbb{E}\|\Pi_i\nabla_\theta\mathcal{L}^i_{\theta,\phi_i}\|^2\le\varepsilon_\phi/(s\,\varsigma_{i})^2$, where $\varsigma_i:=\min_\ell\min\{\varsigma_r(A_{i,\ell}),\varsigma_r(B_{i,\ell})\}$. So \eqref{app:eq:stat} holds up to $\mathcal{O}(\sqrt{\varepsilon_\phi})$ once adapters are well conditioned. At the zero initialization $B_{i,\ell}=0$, the tangent space degenerates and nothing is absorbed. The measured gradient diversity at round 0 therefore coincides with vanilla FedDPO's $\kappa^2$.
\end{remark}
 
\subsection{From Stationarity to Reduced Heterogeneity}
 
\begin{proposition}\label{app:prop:main}
Let $\Pi$ be any orthogonal projector on $\mathbb{R}^d$ whose range lies in the adapted blocks $\theta_{\mathcal{A}}$, and let $\Pi^\perp=I-\Pi$. Define
\[
e_i:=\nabla F_i(\theta)-\Pi^\perp g_i,\;
\delta:=\big\|\{e_i\}\big\|_N,\;
\kappa^2_{\parallel}:=\big\|\{\Pi(\bar g-g_i)\}\big\|_N^2 .
\]
Then
\begin{equation}\label{app:eq:prop}
\kappa_{\mathcal{N}}-\delta\ \le\ \kappa_{\mathrm{eff}}\ \le\ \sqrt{\kappa^2-\kappa^2_{\parallel}}+\delta .
\end{equation}
In particular, $\kappa^2_{\mathrm{eff}}<\kappa^2$ whenever $\delta<\kappa-\sqrt{\kappa^2-\kappa^2_\parallel}$. This threshold is approximately $\kappa^2_\parallel/(2\kappa)$ when $\kappa_\parallel\ll\kappa$.
\end{proposition}
 
\begin{proof}
Averaging the definition of $e_i$ gives $\nabla F=\Pi^\perp\bar g+\bar e$ with $\bar e=\frac1N\sum_ie_i$. Hence
\[
\nabla F-\nabla F_i=\Pi^\perp(\bar g-g_i)+(\bar e-e_i).
\]
\emph{Upper bound.} By the triangle inequality for $\|\cdot\|_N$, $\kappa_{\mathrm{eff}}\le\|\{\Pi^\perp(\bar g-g_i)\}\|_N+\|\{\bar e-e_i\}\|_N$. Since $\Pi$ is an orthogonal projector, $\|\Pi^\perp v\|^2=\|v\|^2-\|\Pi v\|^2$. Hence, the first term equals $\sqrt{\kappa^2-\kappa^2_\parallel}$. For the second, pointwise $\frac1N\sum_i\|\bar e-e_i\|^2=\frac1N\sum_i\|e_i\|^2-\|\bar e\|^2\le\frac1N\sum_i\|e_i\|^2$; taking expectations, it is at most $\delta$.
 
\emph{Lower bound.} On $\theta_{\mathcal{N}}$, $\Pi^\perp=I$, so $(\nabla F-\nabla F_i)_{\mathcal{N}}=(\bar g-g_i)_{\mathcal{N}}+(\bar e-e_i)_{\mathcal{N}}$. Hence $\kappa_{\mathrm{eff}}\ge\|\{(\nabla F-\nabla F_i)_{\mathcal{N}}\}\|_N\ge\kappa_{\mathcal{N}}-\delta$.
\end{proof}
 
Proposition~\ref{app:prop:main} holds for any such projector. Its content lies in choosing $\Pi$ so that $\delta$ is small. The next result identifies what controls $\delta$.
 
\begin{corollary}[Sources of $\delta$]\label{app:cor:delta}
Suppose $\mathcal{L}^i$ has a $\rho$-Lipschitz Hessian, and let $H_i:=\nabla^2\mathcal{L}^i_\theta$. Let $\Pi_i$ be built from $\phi_i^*$. Then
\begin{equation}\label{app:eq:delta}
\|e_i\|\ \le\ \underbrace{\|\Pi_i^\perp H_i\Pi_i\|_{\mathrm{op}}\,\|\Gamma_i^*\|}_{\text{cross-curvature}}
+\underbrace{\tfrac{\rho}{2}\|\Gamma_i^*\|^2}_{\text{higher order}}
+\underbrace{\|(\Pi_i-\Pi)\,g_i\|}_{\text{subspace misalignment}} .
\end{equation}
\end{corollary}
 
\begin{proof}
By Lemma~\ref{app:lem:stat}, $\nabla F_i=\Pi_i^\perp\nabla F_i$, so
\[
e_i=\Pi_i^\perp(\nabla F_i-g_i)+(\Pi_i^\perp-\Pi^\perp)g_i .
\]
By \eqref{app:eq:shift} and Taylor's theorem, $\nabla F_i=\nabla\mathcal{L}^i_{\theta+\Gamma_i^*}=g_i+H_i\Gamma_i^*+q_i$ with $\|q_i\|\le\frac\rho2\|\Gamma_i^*\|^2$. Since $\Gamma_i^*\in\mathrm{range}(\Pi_i)$, we have $\Pi_i^\perp H_i\Gamma_i^*=\Pi_i^\perp H_i\Pi_i\Gamma_i^*$. Finally, $\Pi_i^\perp-\Pi^\perp=\Pi-\Pi_i$.
\end{proof}
 
\paragraph{Interpretation.}
\begin{itemize}
\item \textbf{Shared subspaces.} If the clients' adapters share column and row spaces, so that $\Pi_i=\Pi$ for all $i$, the misalignment term vanishes. If, in addition, the Hessian couples $\mathcal{T}$ and $\mathcal{T}^\perp$ only weakly, \eqref{app:eq:prop} gives
\[
\kappa^2_{\mathrm{eff}}\le\kappa^2-\kappa^2_\parallel+\mathcal{O}(\kappa\,\delta+\delta^2),
\]
where $\delta$ is second order in the adapter shift. This is the precise sense in which ``$\kappa^2_{\mathrm{eff}}\le\kappa^2$ to first order.''
\item \textbf{Misaligned subspaces.} Independently trained adapters do not share subspaces in general. Then $\Pi$ can be taken as the projector onto a common subspace (e.g., $\bigcap_i\mathcal{T}_i$ or a principal common subspace), and the misalignment term is paid explicitly. Outside these conditions, $\kappa^2_{\mathrm{eff}}\le\kappa^2$ is not guaranteed and must be checked empirically.
\end{itemize}
 
\subsection{Role of the LoRA Rank $r$}
 
\paragraph{A ceiling set by which modules are adapted.}
Because $\mathrm{range}(\Pi)\subseteq\theta_{\mathcal{A}}$, we have $\kappa^2_\parallel\le\kappa^2_{\mathcal{A}}$, and \eqref{app:eq:prop} gives the floor $\kappa_{\mathrm{eff}}\ge\kappa_{\mathcal{N}}-\delta$. Adapting only \texttt{c\_attn} can never remove heterogeneity carried by the MLP, embedding, or LayerNorm parameters, regardless of $r$.
 
\paragraph{Growth with $r$.}
Increasing $r$ enlarges $\mathcal{T}_{i,\ell}$, whose dimension is $r(m+n)-r^2$. For nested adapter subspaces, $\kappa^2_\parallel$ is therefore nondecreasing in $r$.
 
\paragraph{Saturation.}
Let $Z_{i,\ell}$ be block $\ell$ of $\bar g-g_i$. A block lies in $\mathcal{T}_\ell$ if and only if $(I-P_\ell)Z(I-Q_\ell)=0$. A sufficient condition is that the column space of $B_\ell$ contains the column space of $Z$. Hence $\kappa^2_\parallel=\kappa^2_{\mathcal{A}}$ as soon as the column space of $B_\ell$ contains $\sum_i\mathrm{col}(Z_{i,\ell})$ for every $\ell$, which is possible when $r\ge\dim\sum_i\mathrm{col}(Z_{i,\ell})$. The rank required is thus governed by the (numerical) ranks of the $N$ mean-zero heterogeneity blocks, not by $d$. Beyond this point, increasing $r$ removes no further heterogeneity.
 
\paragraph{The cost side.}
We hypothesize that larger adapters are harder to bring near stationarity in $E$ local steps. This would raise $\varepsilon_\phi$, and hence the bias $C\varepsilon_\phi$ in Theorem~1, suggesting a non-monotone trade-off in $r$. The SHP rank ablation is consistent with this. On hh-rlhf, the gradient norm decreases monotonically up to $r=16$, which suggests that saturation is not reached in that range.
 
\subsection{DPO-Specific Structure}
 
Consider a linear layer with effective weight $\widetilde W$. Let $x_t$ be its input at token $t$ and $u_t$ the back-propagated signal at its output. Then $\partial\log\pi(\tau)/\partial\widetilde W=\sum_{t\in\tau}u_tx_t^\top$, up to the transpose convention of the implementation. With the implicit-reward margin $h_\theta(\tau^+,\tau^-):=\Delta_\theta(\tau^+,\tau^-)-\Delta_{\mathrm{ref}}(\tau^+,\tau^-)$, the per-client DPO gradient block is
\begin{equation}\label{app:eq:dpograd}
\begin{split}
&\nabla_{\widetilde W}\mathcal{L}^i
=-\beta\,\mathbb{E}_{(\tau^+,\tau^-)\sim\mathcal{D}_i}\Big[\sigma\big(-\beta h_\theta\big)\\&\Big(\textstyle\sum_{t\in\tau^+}u_tx_t^\top-\sum_{t\in\tau^-}u_tx_t^\top\Big)\Big].
\end{split}
\end{equation}
This structure has two consequences for the quantities above.
 
\begin{enumerate}
\item \textbf{Numerical rank of the heterogeneity blocks.} Each $Z_{i,\ell}$ is a weighted sum of token-level outer products over a small local dataset ($|\mathcal{D}_i|=120$ pairs, 64 tokens). When the sigmoid weights concentrate on a few pairs, $Z_{i,\ell}$ is numerically low rank, and a small $r$ suffices for the saturation condition above. This is a hypothesis, and it can be checked directly from the singular values of $Z_{i,\ell}$.
\item \textbf{Source of heterogeneity and the role of $\beta$.} The weight $\sigma(-\beta h_\theta)$ emphasizes pairs with small or negative margin. Clients differ in which pairs these are, and this is where their gradients diverge. As $\beta$ grows, $\sigma(-\beta h_\theta)$ approaches an indicator of misordered pairs ($h_\theta<0$), concentrating the weights further. We conjecture that this lowers the numerical rank of $Z_{i,\ell}$ and increases the absorbed fraction $\kappa^2_\parallel/\kappa^2$. All our experiments use $\beta=0.2$, so this is untested.
\end{enumerate}
 
\subsection{When Is the Reduction Largest?}
 
By Proposition~\ref{app:prop:main} and Corollary~\ref{app:cor:delta}, the gap $\kappa^2-\kappa^2_{\mathrm{eff}}$ is large when:
\begin{enumerate}
\item \textbf{Heterogeneity sits in adapted modules}, i.e., $\kappa^2_{\mathcal{A}}\approx\kappa^2$. If $\kappa^2_{\mathcal{N}}$ dominates, adapting more modules (e.g., MLP projections) helps more than raising $r$.
\item \textbf{The adapted-block heterogeneity lies in the adapters' tangent spaces}, i.e., $\kappa^2_\parallel\approx\kappa^2_{\mathcal{A}}$. This requires $r$ to be large relative to the numerical rank of the $Z_{i,\ell}$.
\item \textbf{$\delta$ is small}: weak cross-curvature, small adapter shifts $\|\Gamma_i^*\|$, and aligned adapter subspaces across clients.
\end{enumerate}
The gap is small when preferences are nearly homogeneous ($\kappa^2\approx0$), when heterogeneity resides mainly in non-adapted parameters, or when adapter subspaces are strongly misaligned.
 
These conditions concern $\kappa^2_{\mathrm{eff}}$ itself. For the algorithm, it additionally matters that adapters are near stationarity (small $\varepsilon_\phi$). Otherwise the base gradients actually transmitted differ from $\nabla F_i$ by the bias $C\varepsilon_\phi$ of Theorem~1, and Lemma~\ref{app:lem:stat} holds only approximately (Remark~\ref{app:rem:approx}).
 
\subsection{Intuition}
 
By \eqref{app:eq:shift}, each client's adapter is a private, low-rank shift of the shared weights. In vanilla FedDPO, every client pushes the shared base in its own direction, and these directions partially cancel under FedAvg, leaving the diversity $\kappa^2$.
 
In pFedDPO, each client first moves its private shift until the loss can no longer be reduced within the adapter's tangent space (Lemma~\ref{app:lem:stat}). Whatever part of the client's disagreement with the population can be expressed there is resolved locally and never reaches the server. What the client uploads is the remainder: the part of its gradient orthogonal to that space, plus non-adapted parameters.
 
The adapter thus acts as a client-specific pre-filter on the base gradient. The filter removes heterogeneity only where it can express it: inside the adapted modules ($\kappa^2_{\mathcal{A}}$) and inside its rank budget ($\kappa^2_\parallel$). It removes it cleanly only when clients' filters are roughly aligned and the curvature does not leak the shift back into the orthogonal complement (small $\delta$).
 
\subsection{Measuring $\kappa^2$ and $\kappa^2_{\mathrm{eff}}$ on the Same Trajectory}
\label{app:kappa:measure}
 
The empirical gradient diversity in Figures~5(b) and~8(b) is computed from pFedDPO's base gradients with the current adapters, i.e., $\widehat\kappa^2_{\mathrm{eff}}:=\frac1N\sum_i\|\frac1N\sum_j\nabla_\theta\mathcal{L}^j_{\theta,\phi_j}-\nabla_\theta\mathcal{L}^i_{\theta,\phi_i}\|^2$. This estimates $\kappa^2_{\mathrm{eff}}$ up to the adapter bias of Assumption~1, but on its own it cannot separate the adapter effect from ordinary training progress.
 
Proposition~\ref{app:prop:main} can be tested directly by computing, at each logging round and at the \emph{same} base $\theta^r$:
\begin{enumerate}
\item[(i)] per-client base gradients with adapters disabled, giving $\kappa^2$;
\item[(ii)] the same gradients with adapters enabled, giving $\widehat\kappa^2_{\mathrm{eff}}$;
\item[(iii)] the block split of (i), giving $\kappa^2_{\mathcal{A}}$ and $\kappa^2_{\mathcal{N}}$;
\item[(iv)] $\frac1N\sum_i\|\nabla_\phi\mathcal{L}^i_{\theta,\phi_i}\|^2$, an estimate of $\varepsilon_\phi$.
\end{enumerate}
The ratio $\widehat\kappa^2_{\mathrm{eff}}/\kappa^2$ then isolates the reduction due to the adapters, and $\kappa^2_{\mathcal{N}}$ gives the floor predicted by \eqref{app:eq:prop}.

\section{Additional Results}\label{sec:additional_results}
\paragraph{Datasets.} Stanford Human Preferences (SHP)~\cite{ethayarajh2021understanding} contains about 385K pairs of human-written Reddit comments answering the same post, drawn from 18 subject-area subreddits; comment $A$ is labeled preferred over $B$ if it received a higher score despite being posted later, so labels reflect the collective votes of Reddit users rather than designated annotators. Anthropic hh-rlhf~\cite{bai2022training} contains multi-turn human--assistant dialogues in which a crowdworker chose the more helpful (helpfulness subsets) or less harmful (harmlessness subset) of two model responses; the chosen and rejected dialogues share the conversation history and differ in the final assistant turn. For each dataset we form $N=5$ agents with 120 disjoint pairs each. Sequences are truncated to 64 tokens, and reward accuracy is computed on training pairs of each agent.

\textbf{SHP Table.} Table~\ref{tab:dpo_comparison} summarizes round-80 results on SHP. Centralized DPO has the highest gradient norm and lowest accuracy, reflecting its smaller per-round gradient work in this (non-compute-matched) regime. Vanilla FedDPO attains the lowest training loss (0.05) and highest reward accuracy (0.82). pFedDPO attains the lowest gradient norm (0.02), with training loss (0.06) and accuracy (0.80) close to vanilla FedDPO. Its measured base-gradient diversity falls from $\sim$1500 to $\sim$150--200. As discussed in the above on measuring $\kappa^2$ and $\kappa^2_{\mathrm{eff}}$, this trajectory alone does not separate the adapter effect from training progress.

\setlength{\tabcolsep}{3pt}

\begin{table*}[htbp]
\centering
\caption{Comparison of DPO variants. Results at round 80 on SHP. Centralized DPO performs worst in this regime: with batch size 4 it receives far less gradient work per round than distributed methods (effective batch $5\times4=20$ for FedDPO), consistent with the computation-dominated analysis in~\cite{jiang2026distributed}. $\kappa^2$ is measured on base model gradients only.}
\small
\begin{tabularx}{\textwidth}{XXXXX}
\toprule
Method & Grad Norm ($\downarrow$) & Train Loss ($\downarrow$) & Reward Acc. ($\uparrow$) & Empirical $\kappa^2$ ($\downarrow$, final) \\
\midrule
Centralized DPO 
  & $\sim$0.70 
  & $\sim$0.10 (unstable) 
  & $\sim$0.75 
  & N/A (no agents) \\
\addlinespace
FedDPO (S=5)~\cite{jiang2026distributed} 
  & $\sim$0.2 
  & $\sim$0.05 
  & $\sim$0.82
  & Not measured \\
\addlinespace
pFedDPO r=8 (ours, default) 
  & $\sim$0.02 
  & $\sim$0.06 
  & $\sim$0.80 
  & $\sim$150–200 (from $\sim$1500) \\
\bottomrule
\end{tabularx}
\label{tab:dpo_comparison}
\end{table*}

\textbf{HH-RLHF.} This appendix reports the full experimental results on Anthropic hh-rlhf using the identical setup as the SHP main results ($N=5$ agents, 120 pairs each, DistilGPT-2, AdamW $\eta=0.0001$, $E=3$, $R=80$ rounds). Figures~\ref{fig:comparative_hh_rlhf}-\ref{fig:participation_hh_rlhf} correspond directly to Figures \ref{fig:comparative_shp}-\ref{fig:participation_shp} in the main text. The qualitative findings are broadly consistent with SHP, with several notable quantitative differences that we highlight below.

\textbf{Comparative Study.}
Figure~\ref{fig:comparative_hh_rlhf} compares pFedDPO, vanilla FedDPO, and centralized DPO on hh-rlhf over 80 rounds. 

\textbf{Gradient norm (panel a).} All three methods begin near 1. pFedDPO descends most rapidly and settles at the lowest floor ($\sim$0.2 by round 80). Vanilla FedDPO (blue) is slightly higher but follows a similarly smooth descent. Centralized DPO (dashed grey) remains the worst throughout, fluctuating around $0.6-1.0$ and never converging stably. This replicates the computation-dominated regime observed on SHP: centralized DPO processes batch size 4 per step while each FedDPO round aggregates across $S=5$ agents (effective batch 20), giving distributed methods substantially more gradient work per round.

\textbf{Training loss (panel b).} pFedDPO and vanilla FedDPO both decrease from $\sim$0.85-0.90 to $\sim$0.10-0.15 by round 80, with pFedDPO reaching a marginally lower final value. Centralized DPO oscillates between 0.3 and 1.4 throughout, never achieving a stable low loss. The instability of centralized DPO is more pronounced on hh-rlhf than on SHP, consistent with its small per-step batch (4 pairs).

\textbf{Preference accuracy (panel c).} All three methods achieve notably higher accuracy on hh-rlhf than on SHP. pFedDPO reaches $\sim$0.95 by round 80, vanilla FedDPO reaches $\sim$0.92, and centralized DPO reaches $\sim$0.80. Absolute accuracies are substantially higher on hh-rlhf than on SHP (all methods $\ge 0.80$ vs.\ $\le 0.82$). We do not attribute this to label quality, since hh-rlhf labels are single crowdworker judgments and SHP labels are inferred from Reddit votes. A more plausible factor is pair structure: hh-rlhf pairs share the dialogue history and differ only in the final assistant turn, and harmlessness pairs often contain lexically distinctive rejected responses, whereas SHP compares two independently written human comments whose quality differences are often subtle.


\textbf{Effect of LoRA Personalization on Client Drift.}
In Figure~\ref{fig:diversity_hh_rlhf} (a), pFedDPO (green) maintains a consistently lower gradient norm than vanilla FedDPO (blue) throughout all 80 rounds on hh-rlhf, replicating the main SHP finding. The gap is visible from the first few rounds and persists to round 80, with pFedDPO settling near 0.2 versus vanilla FedDPO near 0.4-0.6. The magnitude of the improvement is slightly smaller than on SHP, consistent with hh-rlhf’s lower initial $\kappa^2$. 

In Figure~\ref{fig:diversity_hh_rlhf} (b), the empirical gradient diversity $\kappa^2$ on hh-rlhf starts at $\sim$900 at round 0 (compared to $\sim$1500 on SHP), indicating lower inter-client heterogeneity under our hh-rlhf partition than under our SHP partition.
The trajectory is substantially more volatile than on SHP: $\kappa^2$ does not drop cleanly but instead oscillates across the range 400–900 throughout training, with a spike near round 30 and fluctuations of $\pm$ 200–300 around subsequent measurements at rounds 20, 30, 40, 50, 60, and 70. Despite the oscillations, the overall trend is downward, ending near 400–500 by round 70, which is a reduction of approximately 45-55\% from the initial value. This is notably smaller than the 85–90\% reduction observed on SHP, and reflects two factors: (i) lower initial $\kappa^2$ leaves less room for absolute reduction; (ii) the adapters may be further from stationarity after $E=3$ local steps, increasing $C\varepsilon_\phi$, and measuring $\|\nabla_\phi \mathcal{L}^i\|^2$ would test this. Raw $\kappa^2$ values are not directly comparable across datasets because gradient scales differ. $\kappa^2/\|\bar g\|^2$ would be the comparable quantity.

the higher gradient variance of hh-rlhf's pairs makes the adapter stationarity condition (Assumption~\ref{assumption_1}) harder to satisfy within $E=3$ local steps, increasing $C\varepsilon_\epsilon$ and moderating the $\kappa^2_\text{eff}$ reduction.

\textbf{Per-Client Reward Accuracy.} Figure~\ref{fig:per_client_hh_rlhf} shows per-client accuracy for all five agents under pFedDPO on hh-rlhf. All agents improve steadily from $\sim$0.5-0.82 at round 0 to $\sim$0.88-0.97 by round 80, with the average (dashed) reaching approximately 0.90-0.92. Agent 1 (red) is the strongest performer, reaching $\sim$0.95-0.97 by round 80, while Agent 4 (orange) is the slowest, plateauing around 0.86–0.88. The remaining agents (2, 3, 5) cluster between 0.90 and 0.95. 
At round 80 the spread between the highest (Agent 1, $\sim$0.97) and lowest (Agent 4, $\sim$0.88) agents is about 9--10 percentage points, similar to SHP.
The result also confirms that all agents improve under pFedDPO.

\textbf{LoRA Rank Ablation.} Figure~\ref{fig:rank_hh_rlhf} shows the rank ablation $r\in\{2,4,8,16\}$ on hh-rlhf. In panel (a), the accuracy curves are nearly indistinguishable; at round 80 all ranks reach $\sim$0.90--0.92, with $r=8$ marginally highest. Panel (b) shows that the final gradient norm decreases monotonically with rank, from $\sim$0.36 at $r=2$ to $\sim$0.19 at $r=16$. Unlike SHP, where $r=8$ attained the lowest gradient norm, the gradient-norm benefit on hh-rlhf has not saturated by $r=16$ (see ``Role of the LoRA Rank''). Since accuracy is essentially flat across ranks, we keep $r=8$ as the default. On both datasets, rank affects gradient norm more than accuracy, and the best rank appears to be dataset-dependent.

\textbf{Participation Ablation.} In Figure~\ref{fig:participation_hh_rlhf}, it shows the participation ablation $S\in\{1,3,5\}$ for pFedDPO (left) and vanilla FedDPO (right) on hh-rlhf. $S=3$ (blue) achieves the lowest gradient norm, descending to $\sim$0.25 by round 80. $S=5$ (green) is intermediate ($\sim$0.32). $S=1$ (red) is the most volatile, fluctuating around 0.4-0.8 throughout without clear convergence. The $S=5$ vs. $S=3$ gap is moderate; the $S=1$ vs. $S=3$ gap is larger, consistent with the $\zeta^2_g$ variance floor growing sharply as $S$ decreases below the number of agents. 
The gap between participation levels is more compressed for vanilla FedDPO on hh-rlhf.
$S=5$ (green) converges to $\sim$0.6, while $S=3$ (blue) converges to $\sim$0.7. $S=1$ (red) stays near 1 throughout. 
Notably, the difference between $S=5$ personalized pFedDPO ($\sim$0.32) and $S=5$ vanilla FedDPO ($\sim$0.65) is substantial, which is a factor of $\sim2\times$ in gradient norm, and is the clearest evidence on hh-rlhf that personalization provides a meaningful convergence benefit beyond what increased participation alone achieves.

Cross-dataset comparison of participation results: on SHP, the advantage of pFedDPO over vanilla FedDPO at $S=5$ was visible but at $S=1$ both methods struggled similarly. On hh-rlhf, pFedDPO maintains a clear advantage over vanilla FedDPO even at $S=3$. 
One possible explanation is that adapters converge more reliably within $E=3$ local steps on hh-rlhf, reducing $C\varepsilon_\phi$, but we have not measured $\varepsilon_\phi$ directly. With a single seed, the S = 3 vs. S = 5 difference is within run-to-run variation.
\begin{figure*}
    \centering
    \includegraphics[width=1\linewidth]{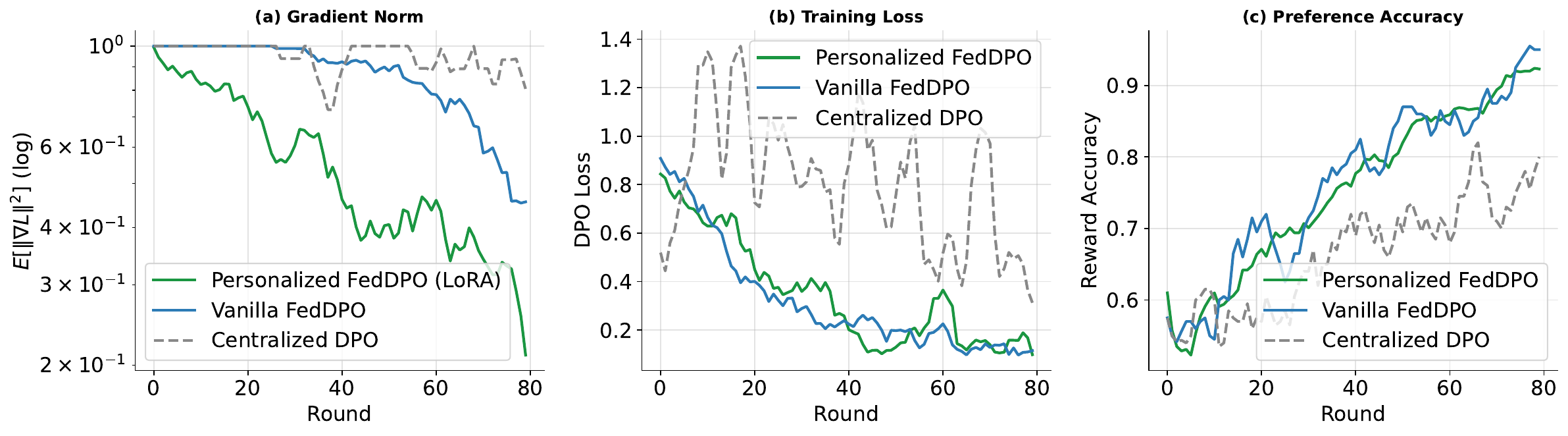}
    \caption{Comparison of baselines with HH-RLHF data.}
    \label{fig:comparative_hh_rlhf}
\end{figure*}

\begin{figure*}
    \centering
    \includegraphics[width=1\linewidth]{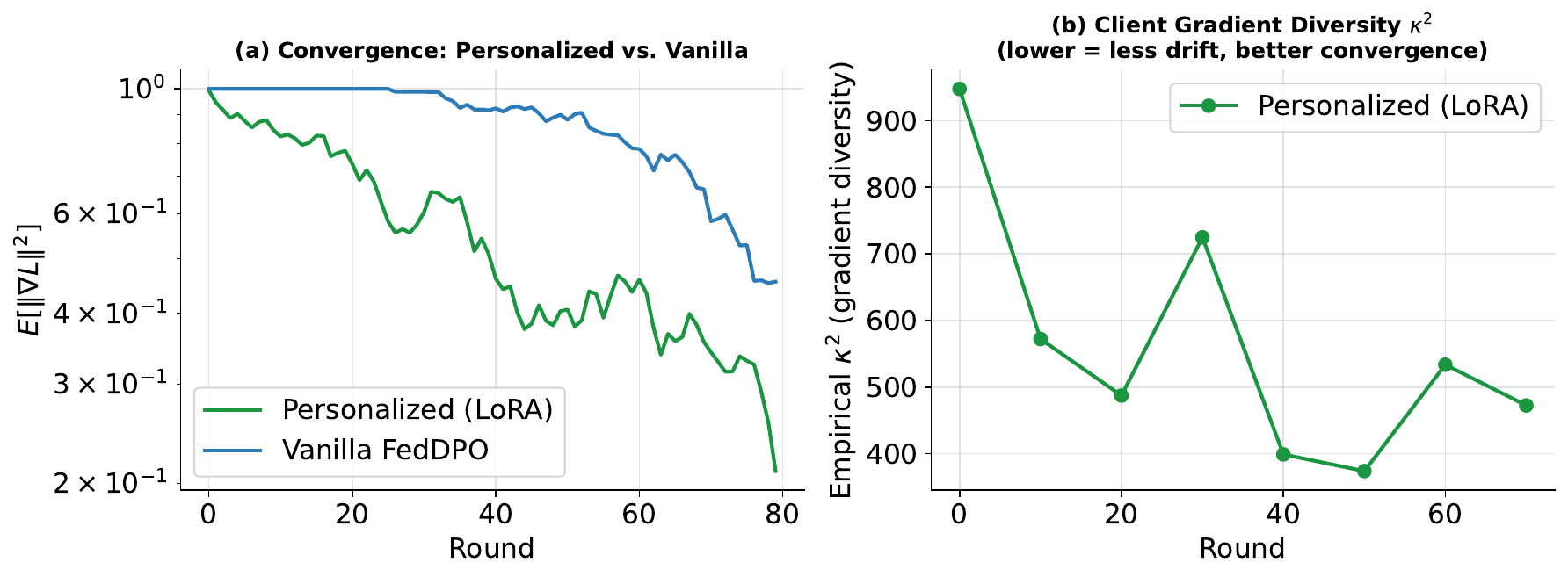}
    \caption{Gradient diversity $\kappa^2$ ablation with HH-RLHF data.}
    \label{fig:diversity_hh_rlhf}
\end{figure*}

\begin{figure*}
    \centering
    \includegraphics[width=1\linewidth]{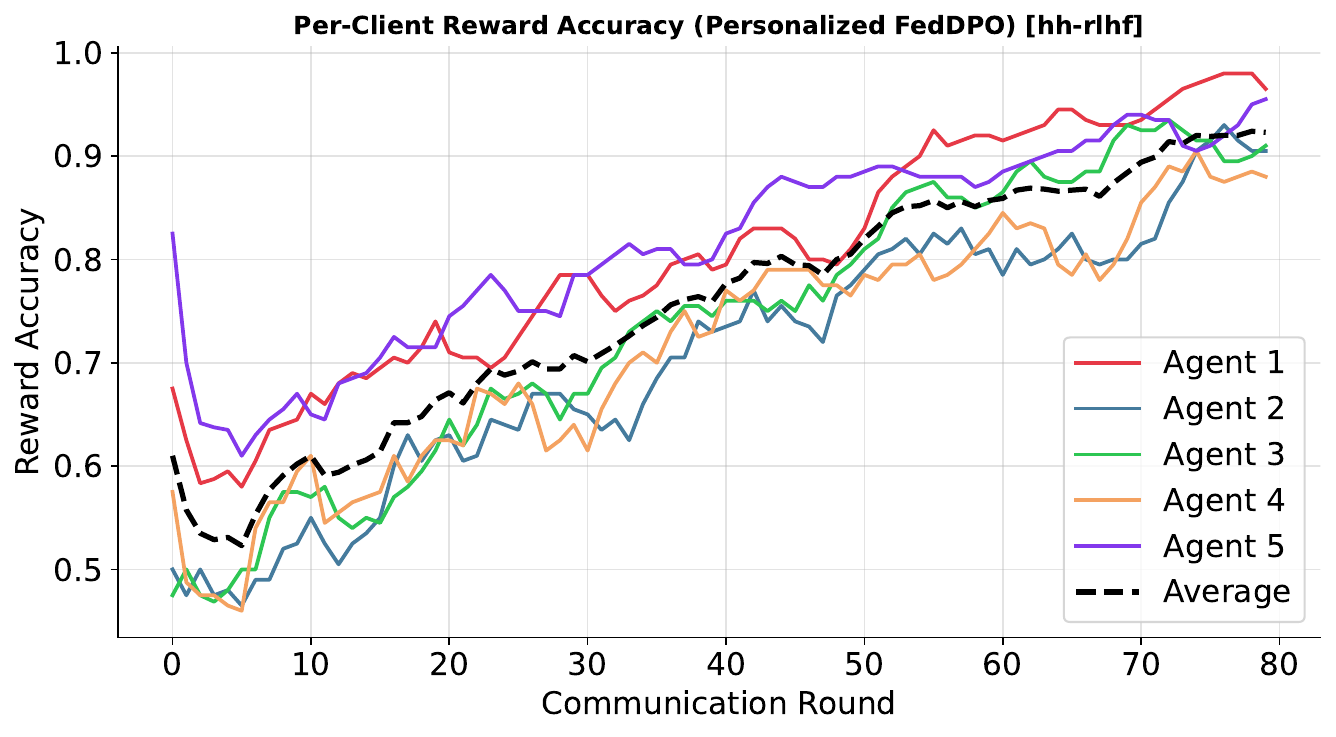}
    \caption{Per-client preference accuracy with HH-RLHF data.}
    \label{fig:per_client_hh_rlhf}
\end{figure*}
\begin{figure*}
    \centering
    \includegraphics[width=1\linewidth]{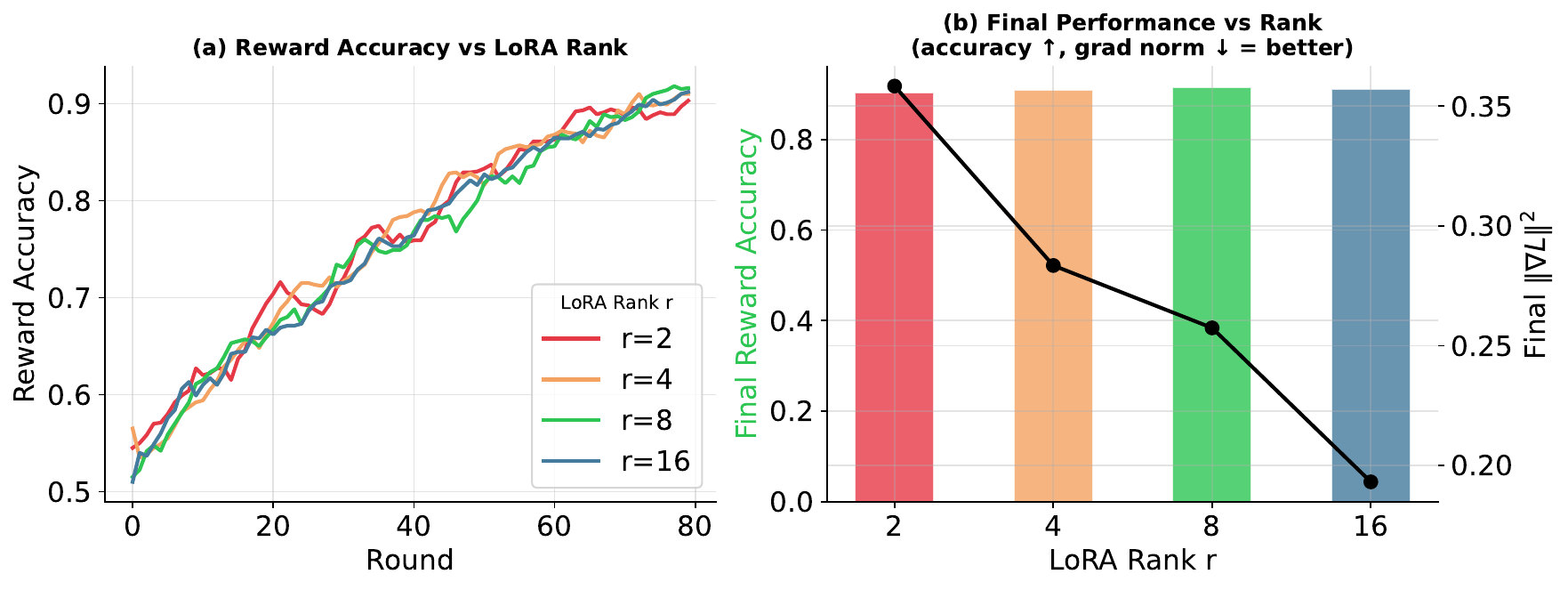}
    \caption{Low rank $r$ ablation with HH-RLHF data.}
    \label{fig:rank_hh_rlhf}
\end{figure*}

\begin{figure*}
    \centering
    \includegraphics[width=1\linewidth]{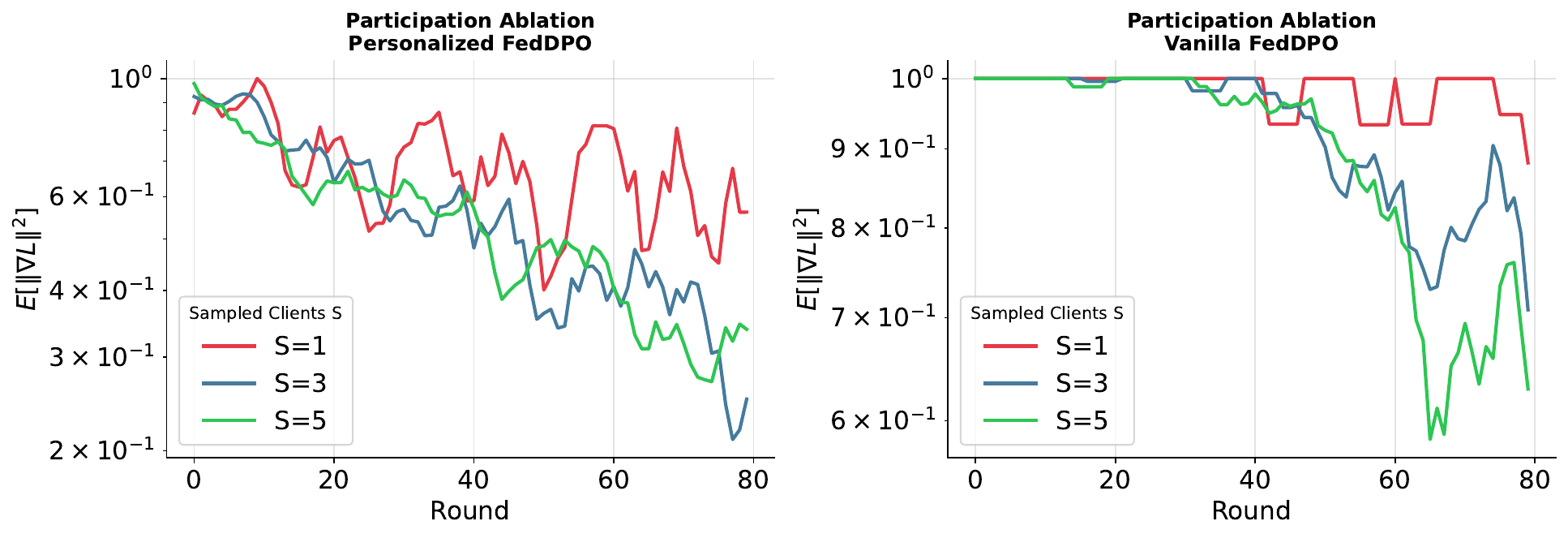}
    \caption{Participation $S$ ablation with HH-RLHF data.}
    \label{fig:participation_hh_rlhf}
\end{figure*}

\section*{Limitations}
Several limitations should be acknowledged. The theoretical analysis inherits the log-linear softmax policy assumptions of (Jiang 2026). Extending the bounds to transformer architectures requires controlling Fisher information spectra and is left for future work. Theorem~1\ref{theorem_1} further assumes that each local loss is $\mu_\phi$-PL in the adapter and that the minimizing low-rank shift is unique, which are strong for bilinear LoRA parameterizations, and $\kappa^2_{\mathrm{eff}}\le\kappa^2$ is established only to first order under a shared-subspace condition. The constant $C$ in the adapter-bias term $C\varepsilon_\phi$ is not available in closed form in terms of $E$, $\eta$, and local curvature; such an expression would let practitioners choose $E$ so that the bound is guaranteed to improve on vanilla FedDPO. pFedDPO does not reduce communication relative to vanilla FedDPO as both transmit $d$-dimensional base updates, whereas FedPDPO \cite{zhu2026fedpdpo} transmits only shared low-rank modules, which may be preferable under tight bandwidth. A variant that transmits low-rank updates while keeping adapters private is an open direction. Empirically, we use DistilGPT-2 ($\approx$82M parameters), 120 randomly partitioned pairs per agent, a single seed for ablations, and report accuracy on training pairs; scaling to 7B+ models, domain-based partitions, held-out evaluation, and multiple seeds are important future work. The rank ablations show that the best rank is dataset-dependent. Finally, pFedDPO assumes synchronous updates, and extending the asynchronous staleness analysis of (\cite{jiang2026distributed}, Thm.~5.4) to private adapters warrants dedicated analysis.

\section*{Ethics Statement}
Federated DPO, including pFedDPO, is designed to enable preference-based alignment of language models in privacy-preserving settings where raw data cannot be centralized. We identify the following ethical dimensions.

\noindent\textbf{Privacy and data minimization.} pFedDPO transmits only base model parameter updates, not raw preference data or LoRA adapters. This avoids sharing raw data and adapter weights, although full base-model updates are higher-dimensional than the low-rank updates shared by adapter-sharing methods and may leak more information. However, model updates can still leak information about local data through gradient inversion attacks. Combining pFedDPO with differential privacy mechanisms is a natural extension that we do not implement here due to the interaction between DP noise and the gradient diversity bound, which warrants dedicated analysis.

\noindent\textbf{Preference heterogeneity and value plurality.} A central motivation for personalization is that different users have genuinely different preferences, and a single global policy may suppress minority preferences by averaging toward the majority. pFedDPO mitigates this by allowing each client’s adapter to capture deviations from the shared base, potentially preserving value diversity. At the same time, private adapters could also encode harmful or biased preferences that would be moderated by global averaging. Deployment should include adapter-level content monitoring or periodic auditing.

\noindent\textbf{Adversarial clients and Byzantine robustness.} pFedDPO, like FedDPO, is vulnerable to clients that submit malicious base model updates to steer the global model. Private adapters cannot be inspected by the server, which may enable model poisoning through the base model update while concealing the intent in the local adapter. Robust aggregation methods should be considered in adversarial settings.

\noindent\textbf{Scope of application.} Our experiments use two open academic preference datasets (hh-rlhf, SHP) for research validation. Deployment of pFedDPO in production systems, particularly in high-stakes domains such as healthcare or legal advice, requires additional safeguards including human oversight of alignment targets, regular evaluation against safety benchmarks, and mechanisms for clients to audit or contest the preference optimization process.


\end{document}